\documentclass{article}
\usepackage[sort&compress,comma,square,numbers]{natbib}
\usepackage{neurodata}
\usepackage{blkarray}
\usepackage{amsmath}
\usepackage[utf8]{inputenc}
\usepackage[english]{babel}
\usepackage{amsthm}
\newcommand{\tab}{\hspace{0.5em}}
\usepackage{lipsum}

\title{Graph Matching via Optimal Transport}

\author[1]{Ali Saad-Eldin\thanks{ali.saadeldin11@gmail.com}} \author[1]{Benjamin D. Pedigo\thanks{bpedigo@jhu.edu}} \author[2]{Carey E. Priebe} \author[1,3]{Joshua T. Vogelstein\thanks{jovo@jhu.edu}}

\affil[1]{Department of Biomedical Engineering, Johns Hopkins University}
\affil[2]{Department of Applied Mathematics and Statistics, Johns Hopkins University}
\affil[3]{Institute for Computational Medicine, Kavli~Neuroscience~Discovery Institute, Johns Hopkins University}

\begin{document}
\maketitle

\begin{abstract}
The graph matching problem seeks to find an alignment between the nodes of two graphs that minimizes the number of adjacency disagreements. Solving the graph matching is increasingly important due to it's applications in operations research, computer vision, neuroscience, and more. However, current state-of-the-art algorithms are inefficient in matching very large graphs, though they produce good accuracy.  The main computational bottleneck of these algorithms is the linear assignment problem, which must be solved at each iteration. In this paper, we leverage the recent advances in the field of optimal transport to replace the accepted use of linear assignment algorithms. We present \textit{GOAT}, a modification to the state-of-the-art graph matching approximation algorithm "FAQ" (Vogelstein, 2015), replacing its linear sum assignment step with the "Lightspeed Optimal Transport" method of Cuturi (2013). The modification provides improvements to both speed and empirical matching accuracy. The effectiveness of the approach is demonstrated in matching graphs in simulated and real data examples. 
\end{abstract}

\section{Introduction}

Graphs are widely used in many fields within the scientific community where data is relational, including social networks, computer vision, and neuroscience \cite{foggia10, conn_code, stat_conn}. In many of these settings, we often work simultaneously with multiple graphs, and want to quantify how they relate to each other. Specifically, we might seek to find a correspondence between the nodes of two graphs such that the connectivity across networks is preserved as best as possible. The Graph Matching Problem consists of finding the bijection between two vertex sets that minimizes the number of adjacency disagreements. If the two graphs are isomorphic, the graph matching problem should find the exact isomorphism between the two graphs. Applications of Graph Matching include biometrics such as facial and fingerprint recognition, video analysis with object and motion tracking, and symbol and string recognition in documentation \cite{foggia30}. 

The graph matching problem is extremely difficult to solve, and no polynomial-time algorithms exist to solve it in its general form. Indeed, in its most general form, the graph matching problem is equivalent to the famous combinatorial optimization problem, the Quadratic Assignment Problem, which is known to be \textbf{NP-hard} \cite{Burkard1998}. For this reason, finding accurate, efficient approximation algorithms for the graph matching problem is an active field of research. There are three main categories of graph matching approximation algorithms: tree search \cite{tree1, tree2}, spectral embedding \cite{umey, rank}, and continuous optimization \cite{faq, sgm, path} methods. In this paper, we focus on graph matching via continuous optimization methods.

The methods presented in this paper focus on the solving the more difficult inexact graph matching problem, where an exact isomorphism may not exist. In general, even though it is known to be in NP-complete, the exact graph matching can be solved relatively accurately via existing methods, as we will show later in this work. However, these methods quickly fail when matching inexact graph pairs. In the past, solutions to matching inexact graph pairs have focused on incorporating prior information, such as seed nodes (a subset of known node correspondences).  Since most real world graph matching applications involve matching non-isomorphic graphs, it is important to be able to be able to match them with no additional information besides the graphs themselves. Additionally, a compounding issue with existing cutting edge graph matching algorithms is that they are slow on large graphs. Owing to their $O(n^3)$ time complexity, they are inefficient on large graphs ($n \approx 10^5$) \cite{divide_and_conquer}.

 We present the algorithm for \textit{Graph matching via OptimAl Transport}, \textbf{GOAT}, a modification to the state-of-the art algorithm, \textit{FAQ}, developed by Vogelstein et. al, 2015 \cite{faq}. Specifically, the proposed algorithm alters the step direction calculation from being a strict permutation matrix to a general doubly stochastic matrix. Since the step direction is computed during the local optimum search in the relaxed sub problem, it is inefficient to restrict the feasible region \cite{opti}. This modification also replaces FAQ's main computational bottleneck of solving the linear assignment problem at each iteration Thus, the modification increases efficiency by decreasing running time on larger graphs. Additionally, we will show how GOAT provides performance advantages over FAQ, specifically adding robustness to stochasticity between graph pairs when graph matching. 

\section{Background}
\subsection{The Quadratic Assignment Problem}
Consider two real matrices $A, B \in \mathbb{R}^{n \times n}$. Let $\mathcal{P} = \{P \in \{0, 1\}^{n \times n} | P \textbf{1}_n = \textbf{1}_n, P^\intercal \textbf{1}_n = \textbf{1}_n \}$ be the set of $n \times n$ permutation matrices, where $\textbf{1}_n$ is the n-dimensional vector of ones. We formally define the Quadratic Assignment Problem (QAP) as the following problem, in matrix notation
\begin{equation}
\tag{QAP}
\begin{aligned}
\min & {\;\trace{A^\intercal PB P^\intercal}}\\
\text{s.t. } & {\;P \: \epsilon \: \mathcal{P}} \\
\end{aligned}
\end{equation}
The combination of quadratic objective function and non-convex feasible region makes this problem \textbf{NP-hard} to solve; no efficient, exact algorithm is known. One could naively solve the problem by finding the objective function value for all permutation matrices in $\mathcal{P}$, however this set is massive, even for small $n$. For reference, $\mathcal{P}$ is of size $n!$, with more than $10^{157}$ solutions when $n=100$. For many applications, $n$ is typically greater than 1,000, meaning that $|\mathcal{P}| > 10^{249}$. We thus seek a method for approximately solving QAP.

Rather than solve the problem over $\mathcal{P}$, we begin by relaxing the feasible to its convex hull, the Birkhoff polytope, also known as the set of doubly stochastic matrices, defined as $\mathcal{D} = \{P \in [0, 1]^{n \times n} | P \textbf{1}_n = \textbf{1}_n, P^\intercal \textbf{1}_n = \textbf{1}_n \}$; that is, square matrices whose rows and columns all sum to one.  We then formally define the relaxed Quadratic Assignment Problem (rQAP) as
\begin{equation}
\tag{rQAP}
\begin{aligned}
\min & {\;\trace{A^\intercal PB P^\intercal}}\\
\text{s.t. } & {\;P \: \epsilon \: \mathcal{D}} \\
\end{aligned}
\end{equation}
Even with the relaxation, rQAP is still non-convex due to it's quadratic objective function not having a necessarily positive definite Hessian. Indeed, when $A, B$ are hollow (matrix diagonal elements are all zero, as is the case with simple graphs with no self-loops), this Hessian is indefinite. However, the convex feasible region allows us to utilize continuous optimization techniques, and thus find a local optima.
\subsection{The Graph Matching Problem}
Consider two graphs, $G_1$ and $G_2$, with vertex sets $V_1, V_2 = \{1, 2, ..., n\}$ and corresponding adjacency matrices $A, B \in \mathbb{R}^{n \times n}$. The Graph Matching Problem (GMP) seeks to find a bijection $\phi: V_1 \rightarrow V_2$ such that the number of edge disagreements between $G_1$ and $G_2$ via $\phi$ is minimized. In matrix notation, the GMP is 
\begin{equation}
\begin{aligned}
\min & {\;|| A - PBP^\intercal||_F^2}\\
\text{s.t. } & {\;P \: \epsilon \: \mathcal{P}} \\
\end{aligned}
\end{equation}
where $||.||_F^2$ is the Frobenius norm. A special case of the GMP is known as a graph isomorphism, when there exists $P \in \mathcal{P}$ such that $A = PBP^\intercal$. Expanding the objective function:
\begin{equation}
\begin{aligned}
|| A - PBP^\intercal||_F^2 = \trace{( A - PBP^\intercal)^\intercal(A - PBP^\intercal)} = \\
\trace{A^\intercal A} + \trace{B^\intercal B} + 2*\trace{APB^\intercal P^\intercal}
\end{aligned}
\end{equation}
Dropping constant terms, (3) is equivalent to
\begin{equation}
\tag{GM}
\begin{aligned}
\min & {\;-\ \trace{A^\intercal PB P^\intercal}}\\
\text{s.t. } & {\;P \: \epsilon \: \mathcal{P}} \\
\end{aligned}
\end{equation}
We see that the objective function for the GMP is just the negation of the objective function for the QAP, and thus any algorithm that solves one can solve the other. 
\subsection{FAQ}

The state-of-art FAQ algorithm of Vogelstein et. al, 2015, uses the Frank-Wolfe algorithm to find a doubly stochastic solution to (2), then solves the linear assignment problem (LAP) to project this solution back onto the set of permutation matrices, thus approximately solving the QAP and GMP \cite{faq}. In 2018, Fishkind et al. proposed modifications to FAQ to extend it's application to allow for seeds (a known bijection subset) and graphs of different sizes \cite{sgm}. The FAQ algorithm is described in Algorithm \ref{FAQ}.

\begin{algorithm}
\caption{FAQ: Find a local optimum to QAP}
\label{FAQ}
\begin{algorithmic}
\Require Adjacency matrices $A, B \in \Real^{n \times n}$.
\State \textbf{Initialize: } $P^{(0)} \in \mathcal{D}$, barycenter ($J = \frac{1}{n}\textbf{1}_n \times \textbf{1}^\intercal_n$) unless otherwise specified
\Linefor{$i = 1, 2, 3, \ldots$ (while stopping criterion not met)}{
\begin{enumerate}
    \item Compute  $ \nabla f(P^{(i)}) = AP^{(i)}B^\intercal + A^\intercal P^{(i)}B$
    \item Compute $Q^{(i)} \in \argmin \tab \trace{Q^\intercal \nabla f(P^{(i)})}$ over $Q \in\mathcal{D}$ via Hungarian Algorithm.
    \item Compute step size $\alpha^{(i)} \in \argmin \tab f(\alpha P^{(i)} + (1 - \alpha) Q^{(i)})$, for $\alpha \in [0, 1]$
    \item Set $P^{(i+1)} = \alpha P^{(i)} + (1 - \alpha) Q^{(i)}$
    \end{enumerate}}\\
\Return $\hat{Q} \in \argmax \tab \trace{Q^\intercal \nabla f(P^{(final)})}$ over $Q \in\mathcal{P}$ via Hungarian algorithm.
\end{algorithmic}
\end{algorithm}

\subsection{Transportation and Sinkhorn Distnaces}
The optimal transport problem is a fundamental probability and optimization problem, in which the transportation of object $\mu$ to $v$ is minimized by some cost. More formally, consider the transportation polytope 
\begin{equation}
\begin{aligned}
U(r,c) = \{P \in \mathbb{R}^{n \times n} | P \textbf{1}_n = r, P^\intercal \textbf{1}_n = c\}
\end{aligned}
\end{equation}

where $U(r, c) \in \mathbb{R}^{n \times n}$, has all non-negative entries, with row and column sums equal to $r$ and $c$, respectively, and $\textbf{1}_n$ is the n-dimensional vector of ones. The optimal transport problem is thus defined as
\begin{equation}
\begin{aligned}
\min & \tab \langle P, M \rangle\\
\text{s.t. } & {\;P \: \in \: U(r, c)} \\
\end{aligned}
\end{equation}
where $M \in \mathbb{R}^{n \times n}$ is the cost matrix \cite{cot}. Cuturi (2013) proposed to modify optimal transport through the addition of a regularizing entropic penalty, $h(P)$ resulting in the \textit{Sinkhorn distance} formulation
\begin{equation}
\begin{aligned}
\min & \tab \langle P, M \rangle - {1 \over \lambda} h(P) \\
\text{s.t. } & {\;P \: \in \: U(r, c)} \\
\end{aligned}
\end{equation}
where $\lambda \in [0, \infty ]$, with the Sinkhorn and transportation distances equivalent for $\lambda$ suitably large. Cuturi showed that Sinkhorn distances could be solved using Sinkhorn-Knopp's fixed point iteration algorithm on $e^{- \lambda M}$ [\ref{LOT}], demonstrating that the method performed very well in practice, and was computationally fast, with an empirical time complexity of $O(n^2)$ \cite{lot,sinkhorn}.

\section{Methods}
\subsection{Motivation}
The primary computational bottleneck in FAQ is the linear assignment problem (LAP) solved in step 2, with the most commonly used linear assignment algorithms (Hungarian and Jonker-Volgenant) having a time complexity of $O(n^3)$ for dense matrices \cite{jv-lap}, \cite{hungarian}. Additionally, even if the adjacency matrix inputs are sparse, the gradient matrix calculated will in practice be dense, and thus sparse LAP solvers will not improve runtime.

Though it is known that the solution to maximizing $\trace{Q^\intercal \nabla f(P^{(i)})}$ over $Q \in\mathcal{D}$ can be chosen to be a permutation matrix (allowing for the use of LAP solvers), we argue that it should not always be selected to be a permutation matrix. This is due to the possibility that this maximization may have multiple solutions (this may also be referred to as a "tie"). Consider the following matrix \\
\[  M =
\begin{blockarray}{ cccc}
\begin{block}{(cccc)}
  40 & 50 & 60 & 65 \\
  30 & 38 & 46 & 48 \\
  25 & 33 & 41 & 43 \\
  39 & 45 & 51 & 59 \\
\end{block}
\end{blockarray}.
 \] 
 Maximizing $\trace{Q^\intercal M}$ over $Q \in\mathcal{P}$ (the feasible region via the linear assignment problem), yields two optimal solutions: \\
 \[  P_1 =
\begin{blockarray}{ cccc}
\begin{block}{(cccc)}
  1 & 0 & 0 & 0 \\
  0 & 1 & 0 & 0 \\
  0 & 0 & 0 & 1 \\
  0 & 0 & 1 & 0 \\
\end{block}
\end{blockarray}
\quad \text{ and} \quad
P_2 =
\begin{blockarray}{ cccc}
\begin{block}{(cccc)}
  1 & 0 & 0 & 0 \\
  0 & 0 & 0 & 1 \\
  0 & 1 & 0 & 0 \\
  0 & 0 & 1 & 0 \\
\end{block}
\end{blockarray}
\]
 
each with objective function value $\trace{Q^\intercal M} = 172$. LAP solvers settle these ties in a arbitrary manner which is sometimes based on the order of the input nodes (as in SciPy's \sct{linear\_sum\_assignment} \cite{scipy} as of this writing). Thus, though the the ordering of the rows/columns of the matrix when input to a LAP solver should have no meaning in our formulation, thus making the final FAQ solution biased to the original bijection between $A$ and $B$. In practice, we avoid such bias by randomly shuffling the nodes of $B$ multiple times at a single initialization. However, such a method can be extremely computationally burdensome, especially when matching large graphs ($n>10,000$).

Additionally, when such a tie is present, we have no reason to choose one solution over another. Indeed, it is easily shown that when multiple optimal solutions exist, all convex combinations of those solutions are also a solution (Lemma \ref{lem1}). We posit, and later show via simulated and real-world data experiments, that choosing a convex combination of these strict permutation matrix solutions can improve the final matching in the GMP, with the additional benefit that solving this relaxed subproblem can often be faster than solving the LAP. We thus seek a method of solving step 2 that is deterministic, while also balancing the possible optimal permutation matrix solutions into a single optimal doubly stochastic matrix solution.







\newtheorem{theorem}{Theorem}
\newtheorem{lemma}[theorem]{Lemma}
\begin{lemma}
\label{lem1}
Let $M \in \mathbb{R}^{n\times n}$ be a matrix with $n$ possible solutions to
$$C = \max_{Q \in \mathcal{P}} \tab \trace{Q^\intercal M}$$ 
where $C$ is the optimal objective function value, and $\mathcal{P}$ is the set of permutation matrices. Denote each solution $P_i$ for all  $i \in \{1,2,...n\}$. Let $P_{\lambda}$ be any convex linear combination of those solutions: 
$$P_{\lambda} = \sum_{i=1}^{n} \lambda_i P_i$$
for any $\lambda = \{ \lambda_1, \lambda_2, ..., \lambda_n \}$ such that $\sum_{i=1}^{n} \lambda_i = 1$ and $\lambda_i \in [0, 1]$.
Then, $P_{\lambda}$ is a solution to 
$$\argmax_{Q \in \mathcal{D}} \tab \trace{Q^\intercal M}$$ 
where $\mathcal{D}$ is the set of doubly stochastic matrices, and furthermore  $C = \trace{P_{\lambda}^\intercal M}$.
\end{lemma}
\begin{proof}
By assumption, $\trace{P_i^\intercal M} = C$. From Birkhoff's Theorem, with the $\lambda_i$'s defined as above, it's known that $\sum_{i=1}^{n} \lambda_i P_i \in \mathcal{D}$ - that is, any convex combination of permutation matrices is a doubly stochastic matrix. Plugging this convex combination into our objective function and using the linearity of the trace, 
$$\trace{\sum_{i=1}^{n} \lambda_i P_i M}= 
\sum_{i=1}^{n} \trace{\lambda_i P_i M} = 
\sum_{i=1}^{n} \lambda_i \trace{P_i M} =
\sum_{i=1}^{n} \lambda_i C  = C$$
\end{proof}


\subsection{Doubly Stochastic Optimal Transport}
Setting $r, c = \textbf{1}_n$, the transportation polytope is equivalent to the Birkhoff polytope, also known as the set of doubly stochastic matrices ($U(\textbf{1}_n, \textbf{1}_n) = \mathcal{D}$). We consider this to be a family of optimal transport, and refer to it as the \textit{doubly stochastic optimal transport problem}. Additionally, since $\langle P, M \rangle = \trace{P^\intercal M}$, the doubly stochastic optimal transport problem can be written as:
\begin{equation}
\begin{aligned}
\min & \tab \trace{P^\intercal M}\\
\text{s.t. } & {\;P \in \mathcal{D}} \\
\end{aligned}
\end{equation}
which is precisely equivalent to the relaxed linear assignment problem (rLAP). Thus, we may use Sinkhorn distances to solve the rLAP. We define the rLAP algorithm inspired by Cuturi as \textit{Lightspeed Optimal Transport} (LOT) in Algorithm \ref{LOT}. To solve the maximization rLAP problem, simply negate M. In practice, we choose $\lambda \geq 100$.
\begin{algorithm}[H]
\caption{LOT}
\label{LOT}
Find doubly stochastic solution to rLAP\\
 \textbf{Inputs:} Cost matrix $M \in \mathbb{R} ^{n\times n}$, $\lambda \in \mathbb{R}$
    \begin{enumerate}
    \item Compute $K \leftarrow e^{- \lambda M}$
    \item Compute $Q \leftarrow Sinkhorn(K)$
    \end{enumerate}
 \Return $Q$
 \label{LOT}
\end{algorithm} 

To demonstrate LOT's effectiveness, we independently realized 100 cost matrices $M \in \mathbb{R}^{n \times n}$ for $n = 250, 500, ...,\\ 3000$, where each entry $M_{i,j} \tab \forall i, j \in n$ is sampled randomly from the $Uniform(100,150)$ distribution (Figure \ref{fig:lap_vs_ot}). LOT results in a substantial speed increase over  traditional LAP algorithms, with little loss in performance, with percent difference in objective function value less than 0.5\%.

\begin{figure}[hbt!] 
    \centering
    \includegraphics[width=\linewidth] {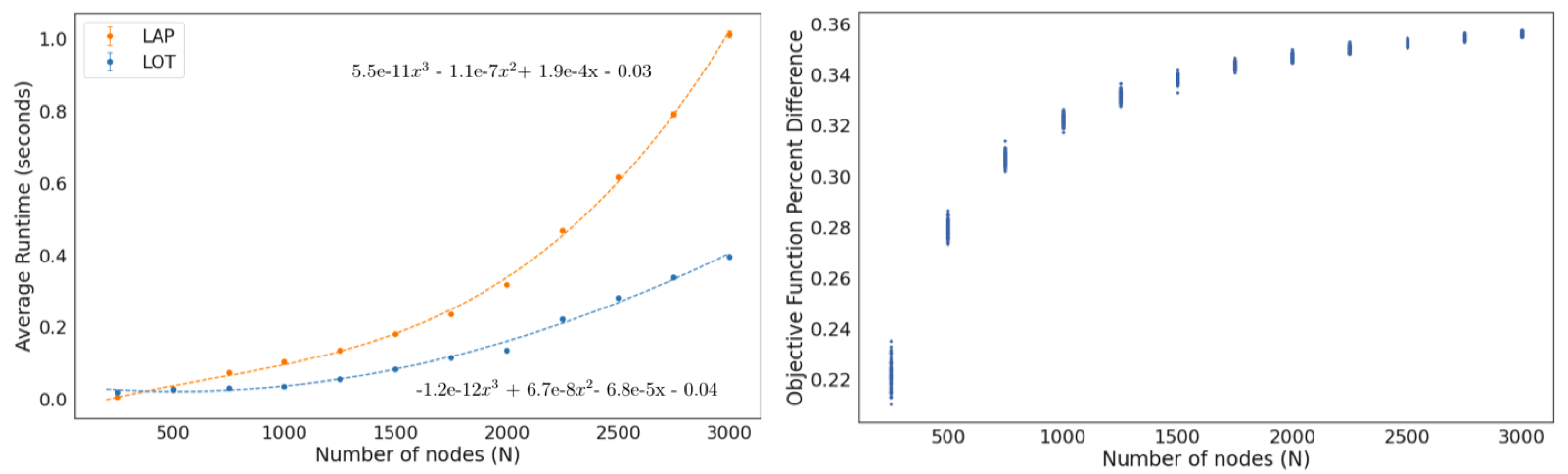}
    \caption{Running time (left) and performance (right) of LAP and LOT as a function of number of nodes, $n$. Each element of cost matrix is independently sampled from a Uniform(100, 150) distribution, with 100 simulations per $n$. Performance defined as relative accuracy, $OFV_{LAP} - OFV_{LOT} \over OFV_{LAP}$, with each dot representing a single simulation.}
    \label{fig:lap_vs_ot}
\end{figure}

\subsection{GOAT}
We introduce the \textit{Graph matching via OptimAl Transport} \textbf{GOAT} algorithm, a modification of the state-of-the-art FAQ algorithm \cite{faq}. 

\begin{algorithm}
\caption{GOAT: Find a local optimum for Graph Matching Problem}
\label{GOAT}
\begin{algorithmic}
\Require Graphs $G_1, \tab G_2$ with vertex sets $V_1, V_2 = \{1, 2, ..., n\}$, and associated adjacency matrices $A, B \in \mathbb{R} ^{n\times n}$
\State \textbf{Initialize: } $P^{(0)} \in \mathcal{D}$, barycenter ($J = \frac{1}{n}\textbf{1}_n \times \textbf{1}^\intercal_n$) unless otherwise specified
\Linefor{$i = 1, 2, 3, \ldots$ (while stopping criterion not met)}{
\begin{enumerate}
    \item Compute $\nabla f(P^{(i)}) = AP^{(i)}B^\intercal + A^\intercal P^{(i)}B$
    \item Compute $Q^{(i)} \in \argmax \tab \trace{Q^\intercal \nabla f(P^{(i)})}$ over $Q \in\mathcal{D}$ via Lightspeed Optimal Transport [\ref{LOT}].
    \item Compute step size $\alpha^{(i)} \in \argmax \tab f(\alpha P^{(i)} + (1 - \alpha) Q^{(i)})$, for $\alpha \in [0, 1]$ as in \cite{sgm}
    \item Set $P^{(i+1)} = \alpha P^{(i)} + (1 - \alpha) Q^{(i)}$
    \end{enumerate}}\\
\Return $\hat{Q} \in \argmax \tab \trace{Q^\intercal \nabla f(P^{(final)})}$ over $Q \in\mathcal{P}$ via Hungarian algorithm.
\end{algorithmic}
\end{algorithm}

In replacing the bottleneck LAP step with the fast and accurate LOT algorithm, GOAT is faster on larger graphs and adds robustness when solving graph matching problems when two graphs are not isomorphic. Rather than tie breaking like LAP solvers, LOT distributes weight across the similar nodes in the doubly stochastic step direction matrix, $Q$.

When initializing GOAT, though any doubly stochastic matrix is feasible, we typically choose the doubly stochastic barycenter, $J = \frac{1}{n} \textbf{1}_n \times \textbf{1}^\intercal_n$ as the initialization. If the graphs are sufficiently small, we can also use several random initializations to maximize performance. Specifically, we independently run GOAT for many $P^{(0)} = \frac{1}{2} (J + K)$, where $K$ is a random doubly stochastic matrix generated via the Sinkhorn-Knopp algorithm, and choose the permutation with the best associated objective function value. 

\section{Results}
We explore the effectiveness of GOAT on simulated and real data examples, measured through matching ratio (the fraction of nodes that are correctly aligned), objective function value, and runtime. Since GOAT is a modification of FAQ, we demonstrate the advantages of GOAT to FAQ in each of our experiments. FAQ has previously been shown to dominate the PATH, QCV \cite{path}, Umeyama \cite{umey}, and RANK \cite{rank} graph matching algorithms. Therefore, by comparing GOAT to FAQ, we can effectively compare GOAT to the other four algorithms as well. For individual experiments, parameters such as maximum number of iterations, stopping tolerance, etc., are consistent across FAQ and GOAT. See code for specific parameter values (Section \ref{code}). 

\subsection{Correlated Network Simulation Setup}
In our simulations, we sample graph pairs $G_1, G_2$ from a $\rho$-correlated Stochastic Block Model (SBM) \cite{sbm, divide_and_conquer}. The $\rho$-SBM model is given:
 \begin{enumerate}
   \item Number of nodes per block, $n \in \mathbb{R}^k$, where $k \in \mathbb{Z}_{>0}$ is the number of blocks.
   \item Edge probability matrix $B \in [0, 1]^{k \times k}$, where $B_{i,j}$ represents the probability of an edge between nodes in community $[i, j]$.
  \end{enumerate}
  
For this model, $\rho = 0$ implies that the graph pairs are independent, and $\rho = 1$ implies that the graph are isomorphic. The $\rho-$SBM model maintains a one-to-one node correspondence across graph pairs, while also incorporating stochasticity. Additionally, the $\rho$-correlated Erdos-Renyi (ER) model can be considered a special case of $\rho$-SBM, in which $k=1$. Here, we utilize the $\rho-$SBM sampler as implemented in graspologic \cite{graspy}.

\subsection{Space and Time Complexity}
As noted earlier, even if the adjacency matrix inputs for GOAT are sparse, the gradient matrix computed will be dense. Thus, GOAT and FAQ share a space complexity of $O(n^2)$, the space required to store $\nabla f(P)$.  We demonstrate the computational advantages of GOAT over FAQ through it's run-time, especially with larger $n$. Since a LAP solver is still required in the final step of the algorithm to project $P^{final}$ onto the set of permutation matrices, as well as the matrix multiplication required, GOAT has a time complexity of $O(n^3)$, equivalent to that of FAQ. However, GOAT has polynomial coefficient terms each an order of magnitude smaller than that of FAQ (Figure \ref{fig:er_time}). We do note however that FAQ is faster for smaller graphs with $n \lessapprox 300$, due to it's constant term of -0.1 compared to GOAT's of 0.1.
 
 Additionally, GOAT far outperforms FAQ in terms of its performance, consistently recovering the exact match between graph pairs for each $n$. Interestingly, FAQ has a faster runtime for  $n<200$ in this simulation, because its constant time term is smaller. Nonetheless, GOAT's performance scales very well, consistently finding the exact match even as $n$ increases.
 \begin{figure}
    \centering
    \includegraphics[width=\linewidth] {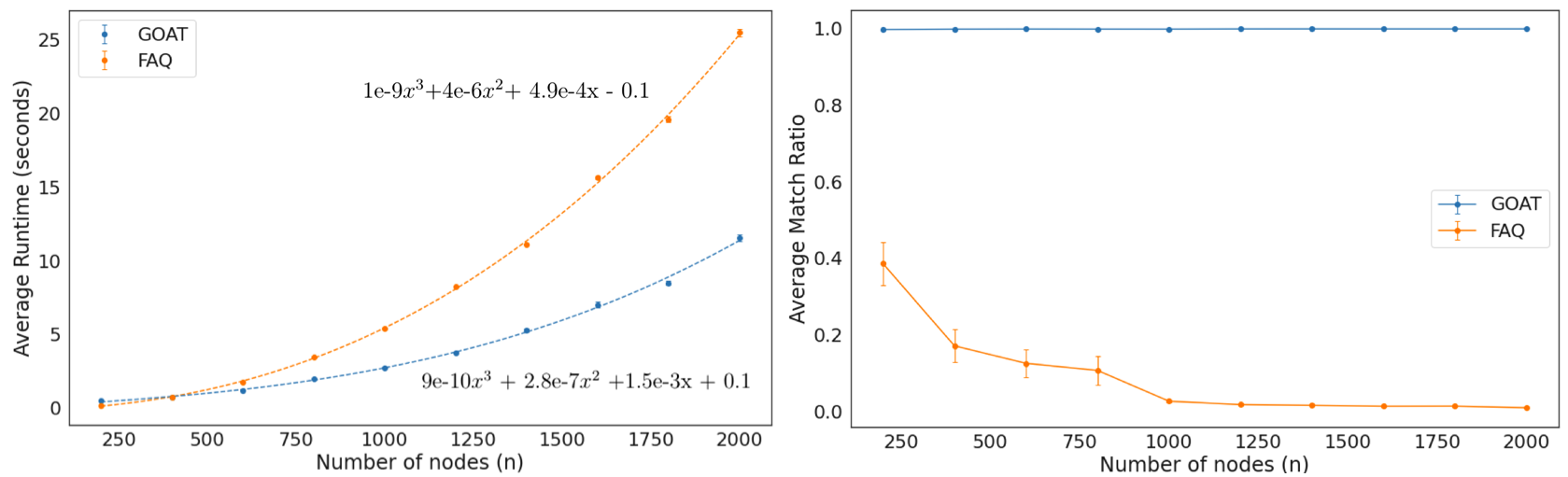}
    \caption{Running time (left) and matching accuracy (right) ($\pm$ s.e) of FAQ and GOAT as a function of number of nodes, $n$. Data sampled from Erdos-Renyi model, with $\rho = {log(n) \over n}$, with 50 simulations per $n$.}
    \label{fig:er_time}
\end{figure}

\subsection{Correlated Network Simulation Accuracy}
Our first simulation result is focused on assessing GOAT's performance in recovering the underlying node correspondence between two graphs. To model a setting in which this correspondence exists, we sample graph pairs $G_1, G_2$ from a $\rho$-correlated SBM.

In our first experiment, we independently sample 100 $\rho$-SBM graph pairs for each value of $\rho = \{0.5, 0.6, \ldots, 1.0\}$ on 150 nodes where $k=3$, and
\\
\[  B =
\begin{blockarray}{ ccc}
\begin{block}{(ccc)}
  0.2 & 0.01 & 0.01 \\
  0.01 & 0.1 & 0.01 \\
  0.01 & 0.01 & 0.2 \\
\end{block}
\end{blockarray},
 \] \\
with each block containing an equal number of nodes. We additionally sample 25 $\rho$-SBM graph pairs for each value of $\rho = \{0.8, 0.85, \ldots, 1.0\}$ on 1,500 nodes, with the same $k$ and $B$ above. For each pair of graphs, we run both FAQ and GOAT, and measure the match ratio, defined as fraction of correctly aligned nodes to total nodes. For each $\rho$ value, we plot the average match ratio over the graph pairs, along with twice the standard error (Figure \ref{fig:rho_sbm}).
\begin{figure}[hbt!] 
    \centering
    \includegraphics[width=\linewidth] {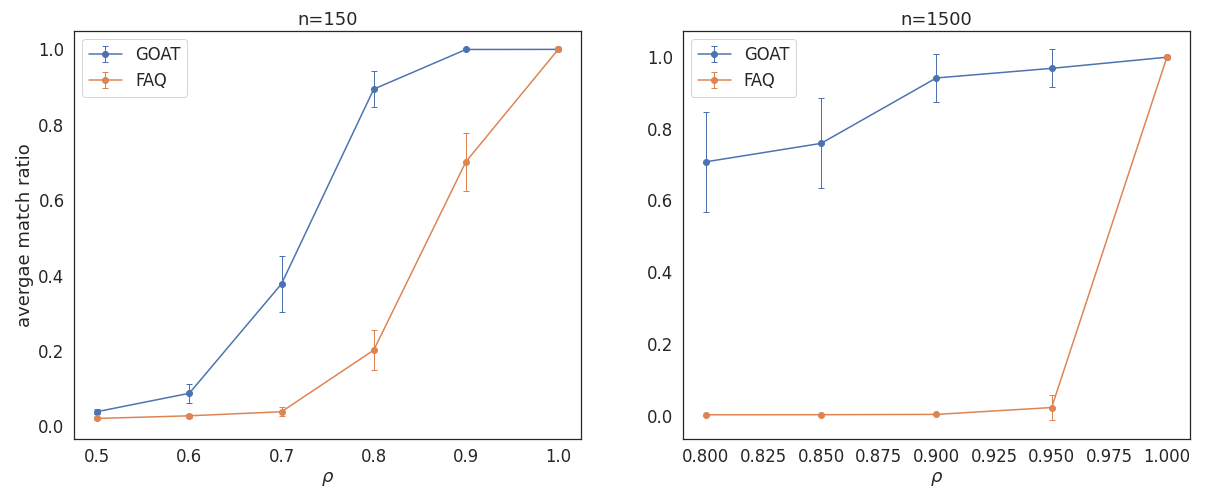}
    \caption{Average match ratio $\pm$ 2 s.e. as a function of correlation values $\rho$ in $\rho$-SBM simulations on $n=150, 1500$ nodes.}
    \label{fig:rho_sbm}
\end{figure}

Across both experiments, we note that GOAT consistently reports a matching accuracy that is greater than or equal to that of FAQ. As expected, GOAT provides improvements to the graph matching results as the strength of the best alignment decreases. It is important to note that in real-data settings, networks are often large and highly correlated, but rarely isomorphic. Indeed, when n = 1,500, the performance gap between FAQ and GOAT is even more noticeable, and when $\rho = 0.95$, GOAT consistently recovers the exact alignment, while FAQ reports an average match ratio $< 0.1$.

In our next experiment, we investigate GOAT's effectiveness in solving the seeded graph matching problem. Applying Fishkind, et. al's modification to both FAQ and GOAT, we run the following experiment, inspired by Figure 2 in Fishkind, et. al, 2018 \cite{sgm}. We independently sample 100 $\rho$-SBM graph pairs for each value of $\rho = \{0.3, 0.6, 0.9\}$ on 300 nodes where $k=3$, and
\\
\[  B =
\begin{blockarray}{ ccc}
\begin{block}{(ccc)}
  0.7 & 0.3 & 0.4 \\
  0.3 & 0.7 & 0.3 \\
  0.4 & 0.3 & 0.7 \\
\end{block}
\end{blockarray}
 \] \\
with each block containing an equal number of nodes. For each rho value, we plot the average match ratio (over the 100 independent realizations) as a function of the number of seeds, $m$.

\begin{figure}[hbt!] 
    \centering
    \includegraphics[width=\linewidth] {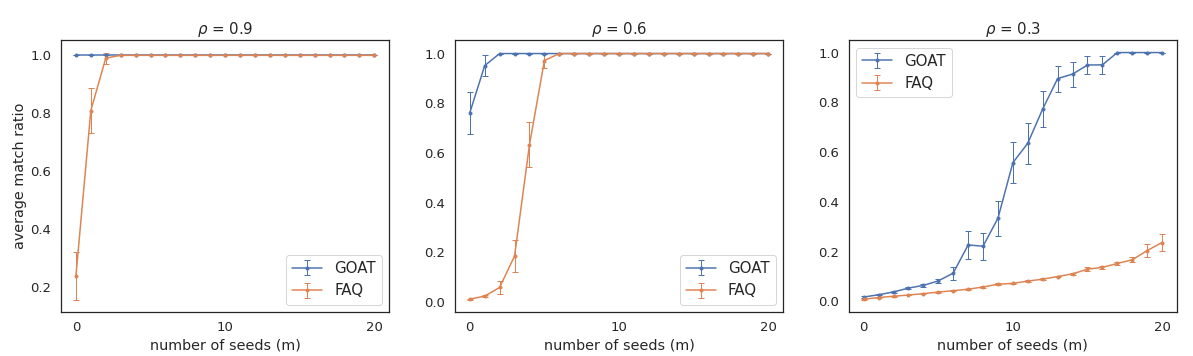}
    \caption{Average match ratio $\pm$ 2 s.e. as a function of number of seeds $m$ for correlation values $\rho = 0.6, 0.9$ in $\rho$-SBM simulations on $n=300$ nodes.}
    \label{fig:sgm}
\end{figure}

In Fig \ref{fig:sgm} we observe that, with all $\rho$ experiments, GOAT requires fewer seeded vertices than FAQ to converge to the exact matching. Additionally, when $\rho = 0.3$, GOAT converges to finding the exact matching, while FAQ does not, with FAQ's matching accuracy less than 0.2 compared to GOAT's accuracy of 1.0.

\subsection{QAPLIB Benchmark}
In order to benchmark GOAT's performance against FAQ's, we evaluate the algorithm's performance on the QAPLIB, a standard library of 137 quadratic assignment problems \cite{qaplib}. Performance is measured by minimizing the objective function $f(P) =\trace{A P B^\intercal P^\intercal}$. In Figure \ref{fig:qaplib}, we plot the log (base 10) relative accuracy $f_{GOAT} \over f_{FAQ}$ for each of the 137 QAPLIB instances, with two initialization schemes:
\begin{enumerate}
    \item Random initialization: the minimum objective function value over 100 initializations at $P = \frac{1}{2} (J + K)$, where $J$ is the doubly stochastic barycenter, and $K$ is a random doubly stochastic matrix.
    \item Barycenter initialization: a single initialization.
\end{enumerate}

Note that a random shuffle on $B$ is performed at each initialization prior to running the algorithms, and random shuffles and initializations are consistent across GOAT and FAQ (the same are used for both methods). 

\begin{figure}[H] 
    \centering
    \includegraphics[width=\hsize] {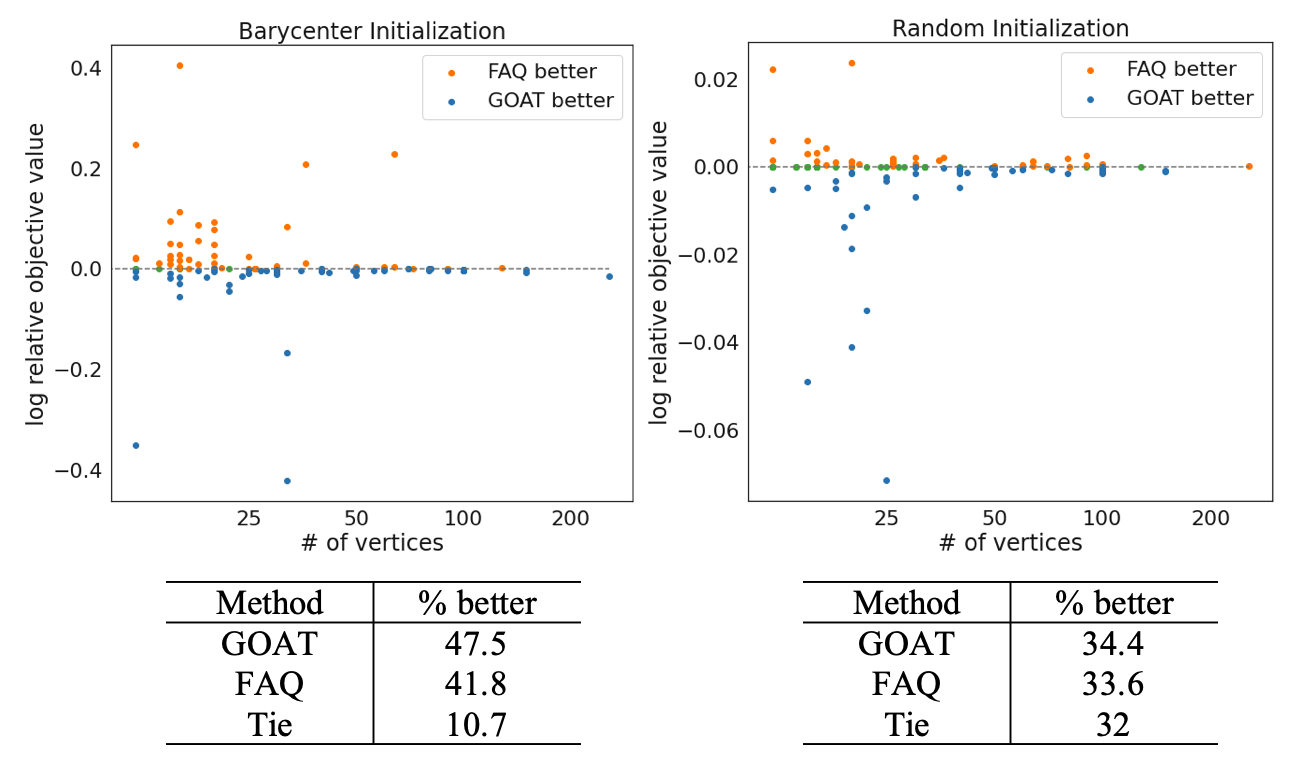}
    \caption{Relative accuracy between FAQ and GOAT, defined as y=log($f_{GOAT} \over {f_{FAQ}}$). Performance is compared using the minimum objective function value over 100 random initializations, and the objective function value of a barycenter initialization with one random shuffle. Initializations and shuffles are the same across FAQ and GOAT.}
    \label{fig:qaplib}
\end{figure}

We note that GOAT performs better on a marginally larger portion of the QAPLIB problems. Using the Mann-Whitney U test, we observe p-values of 0.50 and 0.49 for the random and barycenter initializations, respectively, failing to reject the null hypothesis that GOAT performs differently than FAQ, and vice versa. This result demonstrates that replacing the exact LAP solver with an approximate algorithm causes no significant loss in performance.

\subsection{Organization Email Network Time Series Data}

Next, we analyze an anonymized network time series data set generated from email communication within the Microsoft (MSFT) organization between approximately 80,000 accounts, over a time period of about 10 months, from January 2020 through October 2020. Each month $t$ has its own weighted graph $G_t = (V, E)$ such that each edge $(u, v) \in V \times V$ has weight totaling the number emails between accounts $u$ and $v$ \cite{msr}. Here we demonstrate the real world relevance of the simulation results previously presented, most specifically in Figure \ref{fig:rho_sbm} where GOAT had a sizeable advantage over FAQ when the two graphs were correlated but not isomorphic. 

First, we select a relatively dense induced subgraph on 10,000 nodes from each month $t$ such that there is a 1-1 node correspondence between the subgraphs at each time point. As the time between graph pairs ($\Delta t$) increases, it becomes more difficult to align the two networks. For each $t \in [1, ..., 6]$ (such that $t$ is the month number) and $\Delta t \in [0, ..., 4]$, we match the induced subgraphs taken from $G_t$ and $G_{t + \Delta t}$. To diminish the effect of anomalous edge weights, we apply pass-to-ranks to $G_t$ and $G_{t + \Delta t}$ prior to graph matching \cite{stat_conn}. We observe in Figure \ref{fig:msr_corr} that FAQ immediately fails in solving the graph matching problem when $\Delta t > 0$, with matching accuracy approximately zero, while GOAT still manages to recover a large portion of the true node correspondences. 

\begin{figure}[hbt!] 
    \centering
    \includegraphics[width=\linewidth] {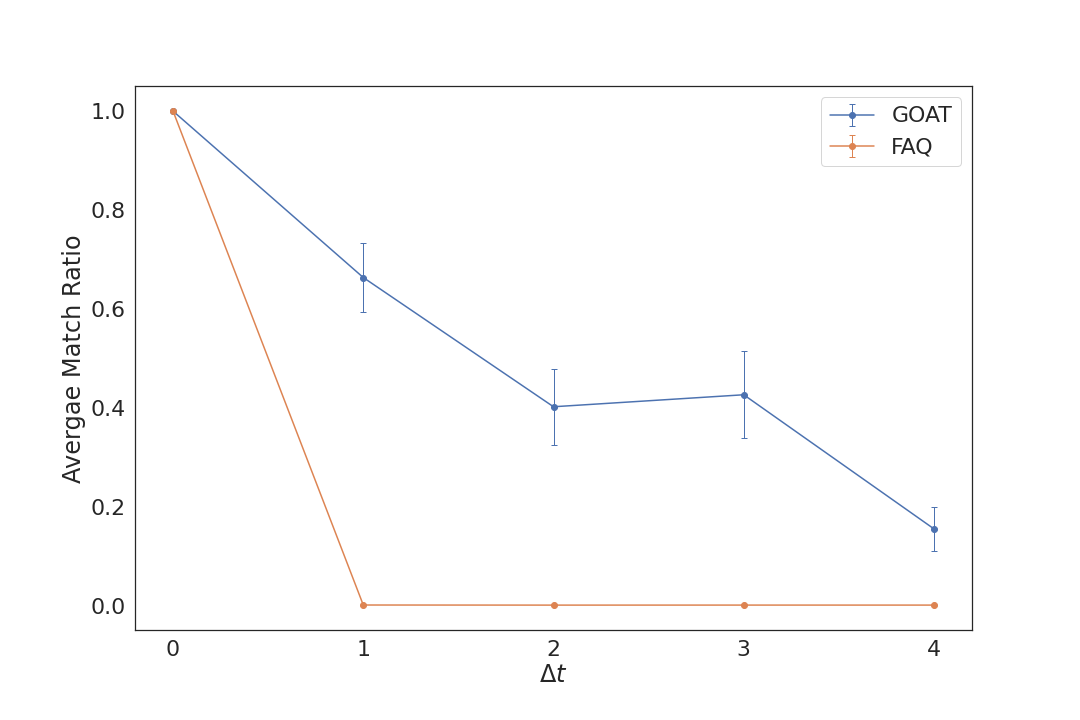}
    \caption{Average match ratio $\pm$ 2 s.e. as a function of $\Delta t$ in MSFT time series induced subgraphs.}
    \label{fig:msr_corr}
\end{figure}



\section{Discussion}
This work presents GOAT, a modification to the state-of-art FAQ algorithm, for solving the graph matching and quadratic assignment problems. The modifications make the algorithm faster on larger graphs, and improve accuracy on simulation benchmarks used in the literature. Most significantly, we demonstrate large improvements for solving inexact graph matching problems, which are most prevalent among real world examples.

In our simulated experiments, we demonstrated that GOAT was faster than FAQ on larger graphs, and consistently far outperformed FAQ when finding accurate bijections across inexact graph pairs. We additionally showed that by adding seeded nodes to the algorithm using the modification specified in \cite{sgm}, the improvements from GOAT still remain. Further, using the QAPLIB benchmarks, which has smaller graphs than those used in simulations, we demonstrated that GOAT still performs as well as FAQ, as was the case with simulated isomorphic graph pairs. Finally, the simulation results were further emphasized by the real world experiment performed on the MSFT time series data, in which the performance benefits for correlated graph pairs were confirmed.

In the future, we hope to extend GOAT's advances to other, related graph matching applications. For example, researchers are often concerned with finding a "soft-matching", or a list of potential matches for each vertex of interest. Previous approaches \cite{sgm, vnsgm} to this so-called vertex nomination problem have performed many restarts of the FAQ algorithm with various initial parameters, and then computed the probability that one vertex was matched to another. Since GOAT deals with soft matchings (i.e. doubly stochastic matrices rather than strictly permutation matrices) at every stage of the algorithm, we suspect that it could be a more natural way to provide vertex nominations. Specifically, we posit that GOAT will allow us to avoid using multiple initializations without loss of accuracy. We leave this investigation to future work. 

\section{Code}
\label{code}
The analysis in this paper were performed using the \texttt{graspologic} \cite{graspy},  \texttt{NumPy} \cite{harris2020array}, \texttt{SciPy} \cite{scipy}, \texttt{Pandas} \cite{mckinney2010data}, and \texttt{POT} \cite{pot} Python packages. Plotting was performed using matplotlib \cite{hunter2007matplotlib} and Seaborn \cite{Waskom2021}. All experiments in this paper can be reproduced using the code at \url{https://github.com/neurodata/goat}, and can be easily viewed at the documentation site \url{http://docs.neurodata.io/goat/outline.html}.


\section*{Acknowledgements}

This research was supported by funding from Microsoft Research. B.D.P. was supported by the NSF GRFP (DGE1746891). This work was supported by the NSF CAREER award (1942963) to J.T.V. and NIH BRAIN Initiative (RFA-MH-19-147) to J.T.V. and C.E.P.

We thank the NeuroData lab for helpful feedback, and Jonathan Larson for help with the network timeseries data.

\clearpage
\bibliographystyle{plainnat}
\bibliography{main.bib}

\end{document}